\documentclass[11pt]{article}

\usepackage{amsmath,amssymb,amsbsy,amsfonts,amsthm,graphicx,url,fullpage}
\usepackage{times,xspace,xcolor}
\usepackage[authoryear]{natbib}

\usepackage{hyperref}
\hypersetup{
colorlinks = true,
citecolor = {blue},
citebordercolor = {white}
}

\usepackage{algorithm}
\usepackage{algorithmic}

\newtheorem{theorem}{Theorem}
\newtheorem{lemma}{Lemma}

\def\shownotes{0}  
\ifnum\shownotes=1
\newcommand{\authnote}[2]{{$\ll$\textsf{\footnotesize #1 notes: #2}$\gg$}}
\else
\newcommand{\authnote}[2]{}
\fi
\newcommand{\yingyu}[1]{{\color{blue}\authnote{Yingyu}{{#1}}}}

\usepackage{multirow}

\newcommand{\veps}{\varepsilon}
\newcommand{\bI}{\bold{I}}
\newcommand{\bM}{\bold{M}}

\newcommand{\bX}{\bold{X}}
\newcommand{\bY}{\bold{Y}}
\newcommand{\bV}{\bold{V}}
\newcommand{\bU}{\bold{U}}
\newcommand{\bSigma}{\bold{\Sigma}}	
\newcommand{\E}{\mathbb{E}}
\DeclareMathOperator*{\argmin}{arg\,min}
\DeclareMathOperator*{\sign}{sign}
\newcommand{\set}[1]{\mathcal{#1}}

\title{Learning Mixtures of Linear Regressions with Nearly Optimal Complexity}
\date{}
\author{Yuanzhi Li\thanks{Princeton University, Computer Science Department, email: yuanzhil@cs.princeton.edu} \and Yingyu Liang\thanks{University of Wisconsin-Madison, Computer Sciences Department, email: yliang@cs.wisc.edu}}

\begin{document}

\maketitle

\begin{abstract}
Mixtures of Linear Regressions (MLR) is an important mixture model with many applications. In this model, each observation is generated from one of the several unknown linear regression components, where the identity of the generated component is also unknown. Previous works either assume strong assumptions on the data distribution or have high complexity. This paper proposes a fixed parameter tractable algorithm for the problem under general conditions, which achieves global convergence and the sample complexity scales nearly linearly in the dimension. In particular, different from previous works that require the data to be from the standard Gaussian, the algorithm allows the data from Gaussians with different covariances. When the conditional number of the covariances and the number of components are fixed, the algorithm has nearly optimal sample complexity $N = \tilde{O}(d)$ as well as nearly optimal computational complexity $\tilde{O}(Nd)$, where $d$ is the dimension of the data space. To the best of our knowledge, this approach provides the first such recovery guarantee for this general setting.
\end{abstract}

\section{Introduction}\label{sec:intro}
This paper studies the problem of learning Mixtures of Linear Regressions (MLR). In this model, one is given i.i.d.\ observations from a mixture of $k$ unknown linear regression components, and the goal is to recover the hidden parameters in the $k$ linear regressions.
In particular, each component $i$ has a sampling probability $p_i$, a data distribution $\set{D}_i$, a hidden parameter $w_i$, and each observation $(x,\alpha)$ is generated by first sampling a component $i$ according to $p_i$'s, then sampling $x$ from $\set{D}_i$ and setting $\alpha = \langle x, w_i\rangle$.

The MLR model is a popular mixture model and has many applications due to its effectiveness in capturing non-linearity and its model simplicity~\citep{de1989mixtures,jordan1994hierarchical,faria2010fitting,zhong2016mixed}. It has also been a recent theoretical topic for analyzing benchmark algorithms for nonconvex optimization (e.g., \citep{chaganty2013spectral,klusowski2017estimating}) or designing new algorithms~(e.g., \citep{chen2014convex}).
However, most of the existing works either restrict to very special settings (e.g., $x$ of different components all from the standard Gaussian, or only $k=2$ components)~\citep{chen2014convex,yi2014alternating,zhong2016mixed,balakrishnan2017statistical,klusowski2017estimating}, or have high sample or computational complexity far from optimal~\citep{chaganty2013spectral,sedghi2016provable}. 

Moreover, to the best of our knowledge, all the existing works require the $\mathcal{D}_i$ being identical. Most works requiring them to be the standard Gaussian, with the exception of those using tensor methods.
However, since the ultimate goal of MLR is to use different linear classifiers to capture different types of data points, it is important to allow different types to have different covariances, and was mentioned as an important open problem in \citep{sedghi2016provable}. 

We propose a novel fixed parameter tractable algorithm for learning Mixtures of Linear Regressions in a setting significantly more general than those in previous works. 
In particular, our setting allows $k \geq 2$ components of data from different distributions $\mathcal{D}_i = \mathcal{N}(0, \bSigma_i^2)$ with $\bI \preceq \bSigma_i \preceq \sigma \bI$, and only requires a necessary separation between the ground truth parameters that any two weight parameters should be at least $\Delta$ apart for some separation parameter $\Delta$.
The algorithm can recover the ground truth to any additive error $\veps$ using 
$N = d \log\left( \frac{d}{\veps}\right) \textrm{poly}\left(\frac{k\sigma}{p_{\min} \Delta}\right)  + n$  
examples and $Nd \cdot \textrm{polylog}(k,d,\sigma,\frac{1}{\veps}, \frac{1}{\Delta}, \frac{1}{p_{\min}})$ computational time, where $p_{\min} = \min_i p_i$ and $n$ is a minor term for fixed $k$. It is tractable in the number of components $k$, the bound on the differences between the different variances $\sigma$, the separation parameter $\Delta$, and the minimum proportion $p_{\min}$ of the components. When these parameters are fixed, it can recover the ground truth to any additive error $\veps$, with nearly optimal sample complexity which is nearly linear in $d$, 
and with nearly optimal computational complexity which is nearly linear in $Nd$. 

Novel algorithmic techniques are proposed since existing ones are not known to generalize to this setting. 
One main technical contribution of our work is a new ``method of moments descent'' technique, that allows us to break ties between different mixture components \emph{gradually}: Unlike most of the previous algorithms which use method of moments to obtain a warm start in one shot, we use it to find a direction to perform one ``gradient descent'' step and gradually refine our solution.  We believe our  techniques are potentially useful in even more general cases.

\paragraph{Organization.} 
Section~\ref{sec:related} reviews the related work, and 
Section~\ref{sec:preli} formalizes the problem and presents our result. An overview of the intuition for designing and analyzing the algorithm is provided in Section~\ref{sec:proofsketch} while the algorithm and the key lemmas are presented in Section~\ref{sec:algo}. The formal proofs are provided in the appendix.

\section{Related Work} \label{sec:related}

Mixtures of Linear Regressions is a popular mixture model (e.g.,~\citep{de1989mixtures,grun2007applications} and \citep{faria2010fitting}), also known as Hierarchical Mixture of Experts in~\citep{jordan1994hierarchical} in the machine learning community. 
It has many applications, such as trajectory clustering~\citep{gaffney1999trajectory} and phase retrieval~\citep{balakrishnan2017statistical}, and has as special cases some popular models, such as piecewise linear regression and locally linear regression.

Learning MLR in general is NP-hard~\citep{yi2014alternating}. Recent interests have been in providing various efficient algorithms for recovering the parameters in MLR under assumptions about the data generation model~\citep{chaganty2013spectral,chen2014convex,yi2014alternating,zhong2016mixed,klusowski2017estimating}. 
They are either under restricted assumptions about the data (mixtures of two component or $x$ all from the standard Gaussian)~\citep{chen2014convex,yi2014alternating,balakrishnan2017statistical,klusowski2017estimating}, or have high sample or computational complexity~\citep{chaganty2013spectral,sedghi2016provable}. 

Some works study specific algorithms for the problem, such as  the Expectation Maximization (EM) algorithm~\citep{khalili2007variable,yi2014alternating,balakrishnan2017statistical,klusowski2017estimating}. It is known that without careful initialization EM is only guaranteed to have local convergence~\citep{klusowski2017estimating}. A grid search method for initialization is proposed in~\citep{yi2014alternating} but is only for the two-component case. It is unclear how to generalize these guarantees to our more general setting where the data $x$ from different components are from different Gaussians.
Moreover, EM also often suffers from a high computational cost.

Another line of works used tensor methods for MLR~\citep{chaganty2013spectral,sedghi2016provable}. The third-order moment is directly estimated in~\citep{chaganty2013spectral} using samples from Gaussian distribution and is estimated from a linear regression problem in~\citep{sedghi2016provable}. A significant drawback of tensor methods is high sample and computational complexity, due to the high cost in estimating and operating over the tensors. 

\citep{chen2014convex} provided a convex relaxation formulation and showed that their algorithm is information-theoretically optimal. However, it is only for the two-component case and suffers from high computational cost in nuclear norm minimization. 

\citep{zhong2016mixed} provided a non-convex objective function that is locally strongly convex in the neighborhood of the ground truth, and proposed to first use a tensor method for initialization and then optimize the provided objective, achieving a global convergence guarantee. The overall algorithm is fixed parameter tractable in the number of components, and achieves nearly optimal sample and time complexity when this parameter is constant. However, it requires all components have the standard Gaussian distribution. It is unclear how to generalize the result to our more general setting where the data $x$ from different components are from different Gaussians. Furthermore, due to the tensor initialization, the algorithm needs complicated assumptions on the moments, while our only essential assumption is that the weight parameters can be separated, which is much simpler and more general (in fact, it is essentially necessary for obtaining any recovery guarantees).

\citep{yi2016solving} gives an improved way of using the tensor method plus alternative minimization so the sample complexity linearly depend on $d$. However, their algorithm  requires that all the data are from the standard Gaussian, and the sample complexity also depends on the minimal singular value of certain moment matrix, which can be $ \Delta^{\Omega(k)}$ small in our setting. 

\section{Problem Definition and Our Result} \label{sec:preli}

In the Mixtures of Linear Regressions (MLR) model, the data $(x, \alpha) \in \mathbb{R}^{d+1}$ is generated by
\begin{align} \label{def:mlr}
  z \sim \mbox{multinomial}(p), ~x \sim \mathcal{D}_z,~ \alpha = \langle w_z, x \rangle
\end{align}
where $p \in \mathbb{R}^k$ is the proportion of different components satisfying $\sum_{i=1}^k p_i=1$, $\mathcal{D}_i$ is the distribution of the $i$-th component, and $\{w_i \in \mathbb{R}^d\}_{i=1}^k $ are the ground truth parameters. The goal is then to recover $\{w_i\}_i$ given a dataset $\{(x_\ell, \alpha_\ell)\}_{\ell=1}^{N}$, where each $(x_\ell, \alpha_\ell)$ is i.i.d. generated by (\ref{def:mlr}).

\paragraph{Notations.} $[k]$ is used to denote the set $\{1, 2, \ldots, k\}$. With high probability or w.h.p.\ means with probability $1 - d^{-C}$ for some sufficiently large constant $C>1$. $1_\set{E}$ is the indicator function of the event $\set{E}$.

\paragraph{Assumptions.} We make the following assumptions about the distributions $\mathcal{D}_i$'s and $w_i$'s.
\begin{enumerate}
\item[\textbf{(A1)}] Each $\mathcal{D}_i = \mathcal{N}(0, \bSigma_i^2)$, where $\bI \preceq  \bSigma_i \preceq \sigma \bI$ for some $\sigma \ge 1$.

\item[\textbf{(A2)}] For every $i \in [k]$, $p_i \ge p_{\min}$ for some $p_{\min} > 0$.

\item[\textbf{(A3)}] Each $\| w_i\|_2 \leq1$, and for some $\Delta \in (0,1)$, 
$
  \| w_i - w_j \|_2 \geq \Delta
$
for any $i \neq j \in [k]$.
\end{enumerate}

Assumption \textbf{(A1)} allows the data $x$ in different components to come from Gaussian distributions with different unknown covariances.\footnote{In the standard linear regression model, the covariance of $x$ can be assumed to be the identity by doing a linear transformation. However, in the mixture of linear regression models, different components have different covariances and thus can not be simultaneously transformed to the identity since which data point comes from which component is unknown.}
This is more general than all the previous works that assume they all come from the standard Gaussian distribution. This also causes difficulties in applying known techniques for MLR, and thus requires new algorithmic approaches.  Moreover, our result can also be easily generalized to the case that the mixtures come from \emph{different} subspaces. That is, there can be zero singular values for $\Sigma_i$'s and the \emph{non-zero} singular values of each component is in $[1, \sigma]$. 

Assumption \textbf{(A2)} controls the imbalance of the components. We should require that there are enough data from each component so that it is possible to recover the corresponding parameter. On the other hand, our technique can also be generalized to the case when there is enough difference between the probabilities. In this case, we could also treat some components as noise and only recover the leading ones. 

Assumption \textbf{(A3)} assumes that the ground truth parameters are separated vectors, which is indeed required for exact recovery. Previous works also in general have some form of separation assumptions, many of which are much more sophisticated than ours (e.g.,~\citep{zhong2016mixed,yi2016solving}). 

\paragraph{Our result.} 
We are now ready to present our result formally.

%
%
%

\begin{theorem}[Main]
\label{thm:main} 
Assume the model~(\ref{def:mlr}) and assumptions \textbf{(A1)}-\textbf{(A3)}. Then Algorithm~\ref{alg:mlr} takes 
$N=d \log\left(\frac{d}{\veps}\right)\cdot \textrm{poly}\left(\frac{k\sigma}{\Delta p_{\min}} \right) +  \left( \frac{\sigma }{\Delta p_{\min} } \right)^{O(k^2)}$ 
data points and in time $Nd \cdot \textrm{polylog}(k, d, \sigma, \frac{1}{\Delta}, \frac{1}{p_{\min}}, \frac{1}{\veps}) $  outputs a set of vectors $\{v_i\}_{i=1}^k$ that with high probability satisfy
$$
  \|v_i  - w_{\pi(i)} \|_2 \le \veps, \forall i \in [k], ~\mbox{for some permutation $\pi$}.
$$ 
\end{theorem}

The theorem shows that the proposed algorithm achieves global convergence. The run time is polylog in $1/\veps$ for recovery error $\veps$, i.e., the algorithm can achieve exact recovery efficiently. Furthermore, in the case where $k, \sigma,$ $p_{\min}$, and $\Delta$ are fixed constants, the sample complexity is nearly linear in the dimension $d$ of the data space, which is nearly optimal in the key parameter $d$. 
The algorithm still works for wider range of $k, \sigma$, $p_{\min}$, and $\Delta$, but with an exponential dependence on $k$. 
%

Table~\ref{tab:previous} shows the comparison with some recent works.
Since for $k=2$ our settings and results subsumes the existing ones, we mainly compare to previous works handling multiple components $k \ge 2$. Algorithms using the tensor method have $\text{poly}(1/\veps)$ dependence~\citep{chaganty2013spectral,yi2014alternating,sedghi2016provable}.
This can be improved by using tensor method only for initialization. 
\citep{zhong2016mixed} provided such an algorithm fixed parameter tractable in the number of components, achieving $N = \tilde{O}(k^k d)$ sample complexity and $\tilde{O}(Nd)$ computational complexity. However, the result is only for the case where the components have data $x$ from the same distribution $\set{D}_i = \mathcal{N}(0, \bI)$. \citep{yi2016solving} provided an algorithm with sample complexity nearly linear in $d$ and polynomial in $k$ but again it is only for the case with $\set{D}_i = \mathcal{N}(0, \bI)$, and furthermore, the sample complexity depends on the minimal singular value of certain moment matrix, which can also be  $\left( \frac{1}{\Delta} \right)^{k}$ small in our setting.  
\citep{sedghi2016provable} provided algorithms for the case where there are $k\geq 2$ components and $\mathcal{D}_i$ are the same (but can be distributions other than Gaussians). It is based on tensor methods and when applied to Gaussian inputs has high sample and computational complexity. 


We also note that it is interesting to compare to results for learning mixture of Gaussians. When the covariance matrix is not axis-aligned, to the best of our knowledge, there is no algorithm for learning  mixture of Gaussians with sample complexity linear in the dimension. Thus, solving the mixture of Gaussian first and then rescale the covariances to identity would clearly fail in our setting. Our result shows how to make use of this small amount of side information (the label $\alpha$) to lower the sample and computational complexity significantly. We refer to for example~\citep{ashtiani2017sample} for some discussions.


\begin{table}
	\centering
\scriptsize
		\begin{tabular}{c| c| c | c}
		\hline
			    & main model assumptions  &  sample complexity $N$   & computational complexity \\
		 \hline
\multirow{2}{*}{\citep{yi2016solving}}  &  $\set{D}_i = \set{N}(0, \bI), k \ge 2$,  separation $\Delta > 0$,  & \multirow{2}{*}{ $\text{poly}(k) \frac{d}{\sigma_k^5 \Delta^2} $ }          &   \multirow{2}{*}{$\text{poly}(k) d^3$ } \\
    & singular value of some moment matrix $\sigma_k$ & & 
\\ \hline
\citep{zhong2016mixed}  &  $\set{D}_i = \set{N}(0, \bI), k \ge 2$, separation $\Delta > 0$          &  $ O(d (k \log(d))^k)$ & $O(Nd \log(d/\veps))$ 
\\ \hline
\multirow{2}{*}{\citep{sedghi2016provable}}  &  $\set{D}_i$ are the same, $k \ge 2$,  &  \multirow{2}{*}{$O\left(\frac{k^4 d^3}{\veps^2 s^2}\right) $ for Gaussian input} & \multirow{2}{*}{much higher than $\tilde{O}(d^2)$ }
\\
& singular values of weight matrix $\ge s>0$ & &
\\ \hline
\multirow{2}{*}{\citep{klusowski2017estimating}} & $\set{D}_i = \set{N}(0, \bI)$, $k = 2$, & \multirow{2}{*}{$\tilde{O}(d)$}  & \multirow{2}{*}{$\tilde{O}(Nd)$ }
\\
& local convergence of EM algorithm  & &
\\ \hline \hline
		 \multirow{2}{*}{Ours}   &  $\set{D}_i = \set{N}(0, \bSigma_i^2), \bI \preceq \bSigma_i \preceq \sigma\bI, k \ge 2$, &  \multirow{2}{*}{$ d \log\left(\frac{d}{\veps}\right) \textrm{poly}\left(\frac{k\sigma}{\Delta} \right)$ + minor term}   & \multirow{2}{*}{$\tilde{O}(Nd)$  }
		\\
		& separation $\|w_i - w_j\| \ge \Delta > 0 (\forall i\neq j)$&  
		\\
		\hline
		\end{tabular}
	\caption{Comparison with some recent related works. Please refer to the papers for details about the model assumptions and dependence on some other less important parameters, which are omitted here for clarity. In particular, the separation parameters in the related work have different meaning from ours and more complicated. \yingyu{k2}}
	\label{tab:previous}
\end{table}
\normalsize

\section{Overview} \label{sec:proofsketch}

For the major part of our paper we will focus on learning the weight for one of the components. This can be iterated straightforwardly to learn all the weights, which will be presented at the end. 

Our algorithm for learning one weight has two phases. In the first phase, we use method of moments to obtain a warm start. In the second phase, we use gradient descent on a \emph{concave} function to get a more accurate solution. 

\paragraph{Method of moments algorithm}
On a high level, our algorithm is based on the following simple strategy: At each iteration $t$, we maintain a vector $a_t$, and the hope is that $\min_{i \in [k]} \{ \| \bSigma_i (w_i - a_t) \|_2\} $ is getting smaller and smaller as $t$ grows, so eventually $a_t$ will be sufficiently close to one $w_i$.  
Since $\alpha - \langle a_t, x \rangle  = \langle x, w_z - a_t \rangle $ comes from a mixture of one dimension Gaussian distributions with variances $\{ \| \bSigma_i (w_i - a_t) \|_2^2\}_{i = 1}^k$, existing algorithms such as~\citep{moitra2010settling} can be used to estimate them. Suppose the next vector $a_{t + 1}$ is simply chosen as $a_t + \eta r$ for a random vector $r \sim \mathcal{N}(0, \bI)$. With at least $1/4$ probability, we know that $r$ is positively correlated with $w_j - a_t $ for $j = \argmin_i \{ \| \bSigma_i (w_i - a_t) \|_2^2\}$, and thus $\| \bSigma_j (w_j - a_t - \eta r) \|_2^2$ will be smaller than $\| \bSigma_j (w_j - a_t  ) \|_2^2$ for sufficiently small $\eta$. If this happens, we can let $a_{t + 1} = a_t + \eta r$ as the next vector. This process is fundamentally different from many of the existing tie breaking algorithms such as~\citep{li2017convergence}, since we do not have any control over which component the algorithm is converging to: the algorithm may switch target components on the fly arbitrarily, but the minimal of $\{ \| \bSigma_i (w_i - a_t) \|_2^2\}_{i = 1}^k$ is always decreasing.

However, this simple strategy is too expensive in terms of the sample and computational complexity. In each iteration, since $r$ is just a random vector, $\| \bSigma_j (w_j - a_t - \eta r) \|_2^2$ can only be smaller than $\| \bSigma_j (w_j - a_t  ) \|_2^2$  for a factor no more than $\frac{1}{d}$. Thus, we need at least $d$ iterations to finish the whole process. Moreover, to guarantee decreasing, we need to estimate $\| \bSigma_i (w_i - a_t  ) \|_2^2$ to accuracy at least $O \left(\frac{1}{d} \right) $ in each iteration, requiring a lot of samples.

The first key idea of our algorithm is to replace sampling from $\mathcal{N}(0, \bI)$ by sampling from $ \mathcal{N}(0, \bU \bU^{\top})$ for some $\bU \in \mathbb{R}^{d \times k}$ whose span is known to contain a vector with good correlation with $\bSigma_j (w_j - a_t  )$. To get this subspace, we rely on the method of moments. Note that 
\begin{align} \label{eqn:moment}
\E[\left( \alpha - \langle a_t, x \rangle \right)^2 x x^{\top}] = \sum_{i = 1}^k p_i    \left( 2\bSigma^2_i (w_i - a_t  )(w_i - a_t  )^{\top} \bSigma_i^2 + \| \bSigma_i (w_i-a_t) \|_2^2\bSigma_i^2  \right).
\end{align}
When all $\bSigma_i = \bI$, we have $\E[\left( \alpha - \langle a_t, x \rangle \right)^2 x x^{\top}]  \propto \bI + \bU \bU^{\top} $ for some $\bU \in  \mathbb{R}^{d \times k}$ whose span is the subspace spanned by $\bSigma_i^2 (w_i - a_t  )$'s. In this case, using a random vector from $\bU$ will make the per-iteration improvement as large as $1/k$, much better than a random vector from the entire space. 

However, such simple process does not carry on to the case when $\bSigma_i$'s are different, since they are reweighed by $ \| \bSigma_i (w_i - a_t  ) \|_2^2  $ in the summation~(\ref{eqn:moment}). As mentioned, we have little control over this reweighing so $\sum_{i= 1}^k p_i \| \bSigma_i (w_i-a_t) \|_2^2\bSigma_i^2$ can be arbitrarily away from $\bI$. 

The second key idea of our algorithm is to combine higher moments with the polynomial method to obtain a good subspace $\bU$. We will use a set of carefully designed coefficients $c_0, \cdots, c_k$ 
such that in the summation $\sum_i c_i \E[\left( \alpha - \langle a_t, x \rangle \right)^{2i} x x^{\top}]$,  the $\bSigma_i^2$ terms will get canceled and all the $\bSigma_i^2 (w_i - a_t  )(w_i - a_t  )^{\top} \bSigma_i^2 $ terms get preserved. 
The $\{c_i\}_{i=0}^k$ are the coefficients of a polynomial constructed to have properties that can ensure the cancellation and preservation. 
More intuition about the construction of this polynomial is given later in Section~\ref{sec:warm_start}. 

We note that many previous algorithms use tensor decomposition as the method of moments gadget (e.g.,~\citep{sedghi2016provable,zhong2016mixed}) to learn the mixtures in one shot. Their algorithms, while being novel and inspiring, either require the data distribution for different components to be spherical Gaussian, or have high complexity to tolerate derivation from spherical Gaussian.

\paragraph{Gradient descent algorithm}
If we only use the method of moments, then we will need $\left(\frac{\sigma}{\veps}\right)^{O(k)}$ sample to achieve error $\veps$. The dependence on $\veps$ is not desired. To achieve the polylog dependence on the final error $\veps$, we only use the method of moments to get a warm start, and then apply gradient descent beginning from the warm start. 

This step is a ``local'' convergence step by using gradient descent to minimize the \emph{concave} function
$$
  g(v) = \E[\log(|\langle w - v, x \rangle| + \zeta)].
$$
Without $\zeta$, the approach is similar to the classical Gravitational allocation~\citep{holden2017gravitational}. However, without it, when $v$ is very close to one of the $w_i$'s,  $\log(|\langle w - v, x \rangle| )$ will be close to zero and becomes less \emph{smooth}. Thus, we add $\zeta$ to ensure smoothness for the convergence of SGD. As we will show, even with a fairly large $\zeta$, SGD will converge \emph{with high probability}. Similar local convergence algorithms were also used in previous works (e.g., \citep{klusowski2017estimating}). However, with our objective function, the proof is \emph{significantly simpler}. 
 
The proof is by lower bounding the correlation between the negative gradient and the difference of the current solution from the ground truth, and then applying standard optimization analysis to get the convergence. The correlation is (a variant) of inverse Gaussians and thus can be bounded; see Section~\ref{sec:gd} for more intuition.


\section{Algorithm} \label{sec:algo}

In this section, we describe our algorithm in three subsections, describing the three parts as mentioned in the overview respectively.


\subsection{Warm Start for Learning One of the Weights} \label{sec:warm_start}

Here we present our algorithm for obtaining a warm start for the weight for one of the components $w_i$, whose
algorithmic ideas and analysis are at the core of this paper.
This algorithm outputs a point $a_T$ such that $\min\{\|a_T - w_i\|_2\}_{i = 1}^k \leq O(\sigma^2\veps)$. The total sample complexity and running time of this algorithm are proportional to $\left(\frac{\sigma}{\veps} \right)^{O(k^2)}$. Eventually, we will take $\veps = \text{poly}\left(\frac{p_{\min}\Delta}{\sigma}\right)$ to enter the warm start for the gradient descent in the next subsection. 

\textsc{MomentDescent} (Algorithm~\ref{alg:one}) describes the details.
It begins with $a_0 = 0$ and iterates to improve it to $a_T$. 
In each iteration, it first uses a set of samples to compute two quantities: $\sigma_t^2$ which is an estimation of $\min \{\|\bSigma_i^2 (w_i - a_t)\|_2\}_{i=1}^k$, and $\bU_t$ which is an estimation of the span of $\{\bSigma_i^2 (w_i - a_t)\}_{i=1}^k$. Then it picks a random vector $v$ from the span of $\bU_t$ and tests if moving $a_t$ along $v$ can decrease $\sigma_t^2$; this is repeated a few times to guarantee success with high probability.

\begin{algorithm}[!t]
\caption{ \textsc{MomentDescent}($k , \delta, \veps$) \label{alg:one}}
\begin{algorithmic}[1]
\REQUIRE  Number of mixture components $k$, failure probability $\delta$, and error $\veps$. 
\ENSURE $a_T$ which is close to some $w_i$ up to error $O(\sigma^2\veps)$ with probability $1 -\delta$. 
\STATE $a_0 \leftarrow 0$. Set $T \leftarrow \Theta(k \sigma \log \frac{\sigma}{\veps}) $ and $q \leftarrow \Theta\left(\log \frac{k\sigma}{\veps\delta} \right)$. 
\FOR{$t = 0, 1, \cdots , T-1$}
\STATE Sample $m  = (\frac{\sigma}{p_{\min}\veps})^{O(k^2)}$
many samples $\{( x_i, \alpha_i) \}_{i = 1}^m $. 
\STATE For every $i \in [m]$, $\alpha_i \leftarrow \alpha_i - \langle x_i, a_t \rangle$.
\STATE  Let $\{ \sigma_i^2\}_{i = 1}^k \leftarrow \textsc{OneDMixture} ( \{ \alpha_i \}_{i = 1}^m, k,  \veps^2/(k\sigma)^2)$.\STATE Let $\sigma_t^2 \leftarrow \min \{ \sigma_i^2\}_{i = 1}^k$. 
\STATE $\bU_t \leftarrow  \textsc{Powerw}(\{ x_i \}_{i = 1}^m, \{ \alpha_i \}_{i = 1}^m, k,  \veps)$
\FOR{$j \in [q]$}
\STATE Pick a random $\gamma \in \mathbb{R}^k$ such that $\gamma \sim \mathcal{N}(0, \bI)$ and let $v = \frac{\bU_t \gamma}{\| \bU_t \gamma \|_2}$. 
\STATE  Sample $m$ many samples $\{ (x_i, \alpha_i) \}_{i = 1}^m $. 
\STATE  For every $i \in [m]$, let $\alpha_i' \leftarrow \alpha_i - \langle x_i, a_t + \eta_t v \rangle$, where $\eta_t = \Theta\left(\frac{\sigma_t}{\sigma\sqrt{k}} \right)$.
\STATE   Let $\{( \sigma_i')^2\}_{i = 1}^k \leftarrow \textsc{OneDMixture} ( \{\alpha_i' \}_{i = 1}^m, k,  \veps^2/(k\sigma)^2)$, 
\STATE Let $(\sigma')^2 \leftarrow \min \{ (\sigma_i')^2\}_{i = 1}^k$
\IF{ $(\sigma')^2  \leq \left(1 - \frac{1}{150 k \sigma}\right) \sigma_t^2 $}
\STATE $a_{t + 1} \leftarrow a_t + \eta_t v$.
\STATE $\bold{break;}$
\ENDIF
\ENDFOR
\ENDFOR
\end{algorithmic}
\end{algorithm}

\begin{algorithm}[!t]
\caption{ \textsc{OneDMixture} ($\{ z_i \}_{i = 1}^m, k, \veps$) \label{alg:1_d}}
\begin{algorithmic}[1]
\REQUIRE $\{ z_i \}_{i = 1}^m$ where each $z_i \in \mathbb{R}$ comes from a mixture of one dimension (mean zero) Gaussian distribution, number of mixture components $k$, and error $\veps$.
\ENSURE $\{\sigma_i^2\}_{i = 1}^k$, the variance of each component up to additive error $\veps$.
\STATE See the algorithm in~\citep{moitra2010settling}. Their theorem implies that the output is up to additive error $\veps$ with $O\left(\frac{\sigma_{\max}}{ p_{\min}\veps } \right)^{O(k)}$ samples, where $\sigma_{\max}^2$ is the maximum variance of those mixtures and $ p_{\min}$ is the minimal probability that one mixture occurs.)
\end{algorithmic}
\end{algorithm}

\begin{algorithm}[!t]
\caption{ \textsc{Powerw}($\{ x_i \}_{i = 1}^m, \{ \alpha_i \}_{i = 1}^m, k,  \veps$) \label{alg:pww}}
\begin{algorithmic}[1]
\REQUIRE  $\{ x_i \}_{i = 1}^m$ where each $x_i \in \mathbb{R}^d$ comes from a mixture of Gaussian distributions, and $\alpha_i$ the label of $x_i$, number of mixture components $k$, and error $\veps$
\ENSURE $\bU \in \mathbb{R}^{d \times k}$, $\veps$ close to the subspace spanned by $\bSigma_1^2 w_1, \cdots, \bSigma_k^2 w_k$
\STATE $\{ \sigma_i^2\}_{i = 1}^k \leftarrow \textsc{OneDMixture} ( \{ \alpha_i \}_{i = 1}^m, k,  \veps^{(g)})$ for $\veps^{(g)} = \left(\frac{\veps}{\sigma} \right)^{4k}  $.
\STATE $\{ c_i \}_{i = 0}^k \leftarrow \textsc{Coeff}(\{ \sigma_i^2 \}_{i = 1}^k, \veps^{(p)})$ for $\veps^{(p)} = \veps$.
\STATE \begin{align}\bM \leftarrow \frac{1}{m} \sum_{p = 0}^k \frac{c_p}{(2p - 1)!!}  \sum_{i = 1}^m \alpha_i^{2p} x_i x_i^{\top}.
\end{align}
\STATE $\bU \leftarrow $ the top-$k$ singular vectors of $\bM$.
\end{algorithmic}
\end{algorithm}

\begin{algorithm}[!t]
\caption{ \textsc{Coeff}($\{ r_i \}_{i = 1}^k, \veps$) \label{alg:coeff}}
\begin{algorithmic}[1]
\REQUIRE  $\{ r_i \}_{i = 1}^k$ where each $r_i \in \mathbb{R}$, and error $\veps$.
\ENSURE $\{ c_i \}_{i = 0}^k$ where each $c_i \in \mathbb{R}$.
\STATE Let $z_1, \cdots, z_s$ be a center of $r_1, \cdots, r_k$ defined by Lemma~\ref{lem:cluster}.
\STATE Let $c_i$ be the coefficient of $x^{2i}$ in the polynomial: \begin{align}
f(x) = \prod_{p = 1}^s (x^2 - z_p).
\end{align}
\end{algorithmic}
\end{algorithm}

\textsc{MomentDescent} uses two subroutines. \textsc{OneDMixture } (Algorithm~\ref{alg:1_d}) is adopted from existing work~\citep{moitra2010settling} and is used to compute $\sigma_t^2$, an estimation of $\min \{\|\bSigma_i^2 (w_i - a_t)\|_2\}_{i=1}^k$. 
So we focus on the other subroutine \textsc{Powerw} (Algorithm~\ref{alg:pww}).

$\textsc{Powerw}$ tries to identify the subspace spanned by $\{\bSigma_i^2 w_i\}_{i=1}^k$, given labels $\alpha_\ell$ from regression weights $\{w_i\}_{i=1}^k$.\footnote{When used in \textsc{MomentDescent}, it is given labels $\alpha_\ell$ from regression weights $(w_i - a_t)$'s, so it will estimate the subspace spanned by $\{\bSigma_i^2 (w_i - a_t)\}_{i=1}^k$.}
As mentioned in the overview, the moments will contain both the signal $\bSigma_i w_i w_i^{\top} \bSigma_i$ and the noise $\bSigma_i^2 $. For example,
\begin{align*} 
\E[\alpha^2 x x^{\top}] = \sum_{i = 1}^k p_i     \left( 2\bSigma_i w_i w_i^{\top} \bSigma_i + \| \bSigma_i w_i \|_2^2 \bSigma_i^2  \right).
\end{align*}
The crucial piece here is to mix the moments with carefully designed coefficients $\{c_p\}_{p=0}^k$, so that $\E[\bM] = \sum_{p = 0}^k \frac{c_p}{(2p - 1)!!} \E[\alpha^{2p} x x^{\top}]$ will mostly contain only the signal.
Later, we will show that if we let $c_p$ to be the coefficients of $z^{2p}$ in some polynomial $f(z) = \prod_{p = 1}^s (z^2 - z_p)$ with carefully chosen $z_1, \cdots, z_s$ that are closely related to $\{ \| \bSigma_i w_i\|_2^2\}_{i = 1}^k$, then 
$$
 \E[\bM] = \sum_{i=1}^k p_i (\bX_i + \bY_i)
$$
where $ \bX_i$ is proportional to $\bSigma_i^2 w_i w_i^{\top} \bSigma_i^2  f'( \| \bSigma_i w_i\|_2)$ and $\bY_i$ is proportional to $\bSigma_i^2 f(\| \bSigma_i w_i \|_2)$.
Therefore, if $j = \argmin_i \| \bSigma_i w_i \|_2$, then we would like $f$ to be small and $f'( \| \bSigma_j w_j\|_2)$ to be large. Furthermore, we would like $f'$ and $f''$ to be bounded to tolerate errors in estimating $\| \bSigma_i w_i \|_2$'s. 

The following lemma shows that such a polynomial can be efficiently constructed. Using this lemma, \textsc{Coeff} (Algorithm~\ref{alg:coeff}) constructs the coefficients $c_p$'s which are used in \textsc{Powerw}. 

\begin{lemma}[Coefficients]\label{lem:cluster}
For every $k \geq 2$, every $\rho > 1$, every $r_1, \cdots, r_k \in [\frac{1}{\rho}, \rho]$, and every $\veps > 0$, one can find in time $O(k\log k)$ an integer $0<s\le k$ and centers $1/\rho \leq z_1 \leq  \cdots \leq z_s \leq \rho$ such that for $f(x) = \prod_{p = 1}^s (x^2 - z_p)$ the following holds.
\begin{enumerate}
\item For $r = \min\{r_i\}_{i = 1}^k$ and every $i \in [k]$, $|f(\sqrt{r_i})| \leq \veps |\sqrt{r}f'(\sqrt{r}) |  $. 
\item $|\sqrt{r}f'(\sqrt{r})| \geq  \left( \frac{ \veps}{\rho} \right)^k$. 
\item For all $x$ with $x^2\in [1/\rho, \rho]$, $|f'(x)| \leq 2k \rho^k$ and $|f''(x)| \leq 4k^2 \rho^k$.
\end{enumerate}
\end{lemma}

Putting things together, we can prove the main lemma regarding the per-iteration improvement of Algorithm~\ref{alg:one}.

\begin{lemma}\label{lem:one_iter}
For every $t \in \{0, 1, \cdots, T - 1\}$ and $\delta > 0$, as long as $\sigma_t = \Omega ( \sigma \veps)$, then with probability at least $1 - \delta$,
$$
\sigma_{t + 1}^2 \leq \left(1 - \frac{1}{200 k \sigma} \right) \sigma_t^2.
$$
\end{lemma}

Using this Lemma and by the choice of our parameters we immediately have the following guarantee for the output of Algorithm~\ref{alg:one}.

\begin{lemma}\label{lem:warmstart}
With probability at least $1-\delta$, $\min_i \|w_i - a_T\|_2 \le O(\sigma^2\veps)$.
\end{lemma}

\subsection{Learning One of the Weights from Warm Start}\label{sec:gd}

\begin{algorithm}[!h]
\caption{ \textsc{GradientDescent}($k , v, \veps$) \label{alg:gd}}
\begin{algorithmic}[1]
\REQUIRE $k$ the number of clusters, a warm start $v$, and the final error $\veps$. 
\ENSURE $v^{(T)}$, recovered weight parameter up to additive error $\veps$.
\STATE Let $v^{(0)} \leftarrow v$, $T \leftarrow \Theta\left(\frac{d}{p_{\min}^2} \log \frac{\zeta}{\veps}\right)$, where $\zeta= \min\left\{\frac{\Delta}{2\sigma}, \frac{\Delta p_{\min}}{64 } \right\}$.
\FOR{$t = 0, 1 , \cdots, T - 1$}
\STATE Sample $m  = \text{poly}\left( \frac{1}{\Delta}, \frac{1}{p_{\min}}, \sigma, \log T\right)$ many samples $\set{S}_{t + 1} = \{ x_i, \alpha_i \}_{i = 1}^m$.
\STATE Update: For properly chosen learning rate $\eta_t = \Theta\left( \frac{\zeta p_{\min}}{d} \right) \times \left(1 - \Theta\left(\frac{p_{\min}^2}{d} \right)\right)^t$ \begin{align}
 v^{(t + 1)} = v^{(t)} + \eta_t \frac{1}{|\set{S}_{t+1}|}\sum_{(x,\alpha) \in \set{S}_{t+1}} \frac{\sign(\alpha - \langle v^{(t)}, x \rangle)}{|\alpha - \langle v^{(t)}, x \rangle| + \zeta} x .
\end{align}
\ENDFOR
\end{algorithmic}
\end{algorithm}

Here we describe how to use gradient descent on a concave function for faster convergence to one of the $w_i$'s, given the warm start computed by the algorithm in the last subsection.

Algorithm~\ref{alg:gd} describes the details. The gradient descent is to minimize the function
$$
  g(v) = \E[\log(|\langle w - v, x \rangle| + \zeta)]
$$
where $\zeta$ is added to make the $\log(\cdot)$ smooth.
The key property used is that we have a large correlation between the negative gradient and the difference of the current solution from the ground truth. Suppose we begin with a warm start close enough to $w_1$, then the correlation is
$ \E \left[\frac{\sign(\alpha - \langle v^{(t)}, x \rangle) \langle w_1 - v^{(t)}, x \rangle}{|\alpha - \langle v^{(t)}, x \rangle| + \zeta} \right]$.
This is (a variant of) inverse Gaussians and can be bounded by a function of the norms $\| w_i - v^{(t)}\|_2$ for $i \in [k]$. Since $\| w_1 - v^{(t)}\|_2$ is much smaller than the other norms $\| w_i - v^{(t)}\|_2$ for $i\neq 1$, the correlation can be shown to be large. 
The convergence then follows from standard analysis.

\begin{lemma}[Gradient descent] \label{lem:gradient}
Suppose there exists $i \in [k]$ such that $\| w_i - v \|_2 \le \zeta / \sigma$. Then with high probability, Algorithm~\ref{alg:gd} outputs a vector $v^{(T)}$ such that
$
  \| w_{i} - v^{(T)} \| \le \veps.
$
\end{lemma}

\subsection{Learning All the Weights}  \label{sec:algo_learn_all}

Here we describe our final algorithm for learning all the weights. It uses the algorithm in the previous subsections to learn the weight of one of the components, removes the data points from that component, and repeats. Note that we can learn the weight up to error $\veps_g$ in time $\log(1/\veps_g)$, so $\veps_g$ can be made as small as $\left(\frac{p_{\min} \Delta}{\sigma d} \right)^{\Omega(k^2)}$ so that the step of removing the data points introduces essentially no error to later steps within our sample size. So we arrive at our final guarantee in Theorem~\ref{thm:main}. 

\begin{algorithm}[!t]
\caption{Learning Mixtures of Linear Regressions\label{alg:mlr}}
\begin{algorithmic}[1]
\REQUIRE Dataset $\set{D} = \{(x_\ell, \alpha_\ell)\}_{\ell =1}^N$, number of components $k$, error $\veps$. (Parameters $\sigma, \Delta, p_{\min}$ are known to all the algorithms)
\ENSURE $\{v_i\}_{i=1}^k$, recovered weight parameters up to additive error $\veps$.
\FOR{$i = 1, \ldots, k$}
\STATE $a \leftarrow$ \textsc{MomentDescent}($k-i+1, \delta, \veps_w$), where $\veps_w = \text{poly}\left(\frac{p_{\min}\Delta}{\sigma}\right)$ and $\delta=\text{poly}\left(\frac{1}{d}\right)$.
\STATE $v_i \leftarrow$ \textsc{GradientDescent}($ k-i+1, a, \veps_g$), where $\veps_g = \min \left\{\veps, \left(\frac{p_{\min} \Delta}{\sigma d}\right)^{\Omega(k^2)} \right\}$.
\STATE Remove from $\set{D}$ all the data $(x_\ell, \alpha_\ell)$ such that $|\langle x_\ell, v_i \rangle - \alpha_\ell| \le \veps_g  \sigma \cdot \text{polylog}(d)$. 
\ENDFOR
\end{algorithmic}
\end{algorithm}

\section{Conclusion} \label{sec:conclusion}

In this paper, we present a fixed parameter algorithm that solves mixture of linear regression under Gaussian inputs in time nearly linear in the sample size and the dimension. Moreover, our sample complexity also scales nearly linear with the dimension $d$. In our setting, we allow each mixture to have a different covariance matrix. Thus, unlike the case when the mixtures are spherical, even the best known algorithm for mixture of general Gaussians would require at least $d^2$ sample complexity to recover the covariance. Our algorithm reduces the sample complexity significantly with the additional one dimensional linear information: it can recover the linear classifier (and thus recover the covariance as well) with $\tilde{O}(d)$ samples.  While the dependency on $d$ is nearly optimal, we would also like to point out that when the total number of mixtures are too large, the sample complexity of our algorithm does suffer from an exponential term of $k$. We believe that with our current set of assumptions, the exponential dependency could be necessary: A lower bound of $e^{k}$ has been proved in~\citep{moitra2010settling} in the very similar setting of learning mixture of Gaussians.

One natural way to get around the exponential dependency is assuming that the covariance $\bSigma_i$ and the hidden vectors $w_i$ satisfies some smoothness assumption (e.g.,~\citep{ge2015learning}). However, the level of smoothness is very subtle in our setting, since the na\"ive application of smoothed analysis often leads to complexity with a large polynomial factor in the dimension. In this paper, near linearity in $d$ is one of our main contributions. We believe that using smoothed analysis while preserving the nearly linear dependency on $d$ is one of the important future directions.

\section*{Acknowledgements}
Yingyu Liang would like to acknowledge that support for this research was provided by the Office of the Vice Chancellor for Research and Graduate Education at the University of Wisconsin –Madison with funding from the Wisconsin Alumni Research Foundation.

\bibliographystyle{plainnat}

\bibliography{bibfile}

\appendix
\section{Proof of Warm Start for Learning One of the Weights}
\label{sec:proof_one}

We prove the following lemma related to the output of Algorithm~\ref{alg:one}.

\medskip
\noindent
\textbf{Lemma~\ref{lem:warmstart}}
{\it
With probability at least $1-\delta$, $\min_i \|w_i - a_T\|_2 \le O(\sigma^2\veps)$.
}
\medskip

Before proving this lemma, we first need the following lemma about the clustering, which is crucial for constructing the coefficients. As we shall see, we will use this lemma on $r_i = \| \bSigma_i (w_i - a_t) \|_2^2$. Roughly speaking, $f(\sqrt{r_i})$ is the weight of $\bSigma_i^2$ and $f'(\sqrt{r_i})$ is the weight of $\bSigma_i^2 (w_i - a_t) $. Therefore, we would like $f(\sqrt{r_i})$ to be small compare to  $f'(\sqrt{r_i})$ to identify the subspace spanned by $\bSigma_i^2 (w_i - a_t) $. 



\medskip
\noindent
\textbf{Lemma~\ref{lem:cluster} (Coefficients)}
{\it
For every $k \geq 2$, every $\rho > 1$, every $r_1, \cdots, r_k \in [\frac{1}{\rho}, \rho]$, and every $\veps > 0$, one can find in time $O(k\log k)$ an integer $0<s\le k$ and centers $1/\rho \leq z_1 \leq  \cdots \leq z_s \leq \rho$ such that for $f(x) = \prod_{p = 1}^s (x^2 - z_p)$ the following holds.
\begin{enumerate}
\item For $r = \min\{r_i\}_{i = 1}^k$ and every $i \in [k]$, $|f(\sqrt{r_i})| \leq \veps |\sqrt{r}f'(\sqrt{r}) |  $. 
\item $|\sqrt{r}f'(\sqrt{r})| \geq  \left( \frac{ \veps}{\rho} \right)^k$. 
\item For all $x$ with $x^2\in [1/\rho, \rho]$, $|f'(x)| \leq 2k \rho^k$ and $|f''(x)| \leq 4k^2 \rho^k$.
\end{enumerate}
}
\medskip

\begin{proof}[Proof of Lemma \ref{lem:cluster}]
%
Let us without loss of generality assume that $r=r_1 \leq r_2 \leq \cdots \leq r_k$.  Let us define $z_1 = r_1$, and let $j \in [k ]$ be the smallest  index such that $r_j \geq z_1 + \frac{\veps}{{\rho}}$. If no such index exists, we let $s = 1$ and the statements in the lemma are true.  If such $j$ exists, let us define:
\begin{align}
z_2 = r_j, z_3 = r_{j + 1}, \cdots, z_s = r_k.
\end{align}

Now, we know that 
\begin{align}
|\sqrt{r}f'(\sqrt{r})| &= 2 r \prod_{p = 2}^{s} |r - z_p| \geq  \left( \frac{ \veps}{{\rho}} \right)^k.
\end{align}

On the other hand, for every $i \geq j$, $f(\sqrt{r_i}) = 0$. For $i < j$ we have: 
\begin{align}
|f(\sqrt{r_i})| &=  |r_i - r|\prod_{p = 2}^{s} |r_i - z_p|
\\
&\leq \frac{\veps}{{\rho}} \prod_{p = 2}^{s} |r_i - z_p| \leq \veps r \prod_{p = 2}^{s} |r - z_p|  \leq \veps  |\sqrt{r} f'(\sqrt{r})|.
\end{align}

We now consider the derivative and second order derivative of $f(x)$ for $x^2 \in [0, \rho]$. By elementary calculation, we know that 
\begin{align}
|f'(x)| &= \left|\sum_{p = 1}^s 2x \prod_{q \not= p} (x^2 - z_q) \right|
\\
& \leq 2\sum_{p = 1}^s |x |  \prod_{q \not= p} \left| x^2 - z_q \right|
\\
& \leq 2 k \rho^k.
\end{align}

Similarly we can get that $|f''(x)| \leq 4k^2 \rho^k$. 
\end{proof}

We also need the following bound for the $k$-SVD of a matrix.

\begin{lemma}\label{lem:k_SVD}
Let $\bX_1, \cdots, \bX_k$ be $k$ rank-one matrices in $\mathbb{R}^{d \times d}$ such that each $\bX_i = x_i x_i^{\top}$, for every $\veps \geq 0$, every PSD matrix $\bM \in \mathbb{R}^{d \times d}$ such that 
\begin{align}
\left\| \bM - \sum_{i = 1}^k \bX_i \right\|_2 \leq \veps \| \bX_1 \|_2
\end{align}
Let $\bU \in \mathbb{R}^{d \times k}$ be the matrix consists of the top-k singular vectors of $\bM$, then we have
\begin{align}
\| x_1^{\top}\bU \|_2 \geq \left(1 - (\veps k)^{1/3}\right)\| x_1 \|_2
\end{align}

\end{lemma}
\begin{proof}[Proof of Lemma \ref{lem:k_SVD}]
Let us denote $\sigma_1 \geq  \cdots \geq \sigma_k \geq \sigma_{k + 1} =  0$ as the $k + 1$ singular values of $\sum_{i = 1}^k \bX_i$ with corresponding singular vectors $v_1, \cdots, v_k$ (and $v_{k + 1}$). For every $v_i$, by definition
\begin{align}
v_i^{\top} \left(\sum_{j = 1}^k \bX_j \right)v_i = \sigma_i
\end{align}
So we have $v_i^{\top} \bX_1 v_i  \leq \sigma_i$. Let $\bV_i  \in \mathbb{R}^{d \times i}$ defined as $\bV_i= (v_1, \cdots, v_i)$. By Gap-free Wedin theorem in~\citep{allen2016lazysvd} (see Lemma~\ref{lem:gapfree_wedin}), we know that 
\begin{align}
\| (\bI - \bU \bU^{\top}) \bV_i \|_2 \leq \frac{\veps \| x_1 \|_2^2}{\sigma_i}.
\end{align}
Thus, $\|x_1^{\top} (\bV_i  \bV_i^{\top})(\bI - \bU \bU^{\top}) \|_2 \leq \frac{\veps \| x_1 \|_2^3}{\sigma_i}$. 

On the other hand, since $x_1 \in \text{span}\{v_1, \cdots, v_k \}$, 
\begin{align}
\|x_1^{\top} (\bI - \bV_i  \bV_i^{\top}) \|_2 &= \|x_1^{\top} (\bV_{k} \bV_{k}^{\top} - \bV_i  \bV_i^{\top}) \|_2 
\\
&\leq  \sum_{j = i + 1}^k |x_i^{\top} v_k| \leq k \sqrt{\sigma_{i + 1}}.
\end{align}
Therefore, we know that 
\begin{align}
\| x_1^{\top}  (\bI - \bU \bU^{\top}) \|_2 \leq \frac{\veps \| x_1 \|_2^3}{\sigma_i} + k \sqrt{\sigma_{i + 1}}.
\end{align}

If $\sigma_1 \geq  \frac{\| x_1\|_2^2 \veps^{2/3}}{k^{2/3}}$, by picking $i$ to the largest index in $[k]$ such that $\sigma_i \geq \frac{\| x_1\|_2^2 \veps^{2/3}}{k^{2/3}}$,  we get that 
\begin{align}
\| x_1^{\top}  (\bI - \bU \bU^{\top}) \|_2 \leq (\veps k)^{1/3} \| x_1 \|_2
\end{align}

If $\sigma_1 \leq  \frac{\| x_1\|_2^2 \veps^{2/3}}{k^{2/3}}$, then we can just use $\|x_1^{\top}  \|_2 \leq k \sqrt{\sigma_1}$ to complete the proof.
%
%
%
%
%
%
\end{proof}

We are now ready to prove the following important lemma about the correlation between $\bU$ and $\bSigma_i^2(w_i - a_t)$.

\begin{lemma}\label{lem:correlation}
Let $j = \argmin_{1\le i\le k}  \|\bSigma_i (w_i - a_t) \|_2$, we have that in the $t$-th iteration of Algorithm~\ref{alg:one}, the $\bU_t$ satisfies
\begin{align}
\frac{\|\bU_t^{\top} \bSigma_j^2 (w_j - a_t) \|_2}{\| \bSigma_j^2 (w_j - a_t)  \|_2 } \geq \frac{1}{2}.
\end{align}

\end{lemma}
\begin{proof}[Proof of Lemma \ref{lem:correlation}]
Suppose $z \sim \mathcal{N}(0, \bSigma^2)$, we know that $z = \bSigma g$ where $g \sim \mathcal{N}(0, \bI)$. For every vector $a$,
\begin{align}
\E\left[\langle z, a \rangle^{2p} z z^{\top}\right] &= \bSigma \E\left[ \langle g, \bSigma a \rangle^{2p} g g^{\top}\right]  \bSigma
\\
&=  (2p - 1)!! \bSigma \left( 2p  \bSigma a a^{\top} \bSigma \| \bSigma a\|_2^{2p - 2}  + \| \bSigma a \|_2^{2p} \bI \right) \bSigma
\\
& =   (2p - 1)!!  \| \bSigma a\|_2^{2p} \left( 2p \frac{ \bSigma^2 a a^{\top} \bSigma^2 }{ \|\bSigma a \|_2^2 }+  \bSigma^2 \right).
\end{align}
Thus, we have
\begin{align}
\frac{1}{(2p - 1)!!}\E\left[ \alpha_i^{2p} x_i x_i^{\top} \right] &= \sum_{i = 1}^k p_i   \| \bSigma_i (w_i - a_t)\|_2^{2p} \left( 2p \frac{ \bSigma_i^2 (w_i - a_t) (w_i - a_t)^{\top} \bSigma_i^2 }{ \|\bSigma_i (w_i - a_t) \|_2^2 }+  \bSigma_i^2 \right).
\end{align}

Since in the $t$-th iteration, the labels $\alpha_i$ we fit to Algorithm~\ref{alg:pww} comes from $\alpha_{\ell} = \langle x_{\ell}, w^{(\ell)} - a_t \rangle$, we know that 
\begin{align}
\E[\bM] =  \sum_{i = 1}^k p_i    \sum_{p = 0}^k \left( c_p \| \bSigma_i (w_i - a_t)\|_2^{2p} \left( 2p \frac{ \bSigma_i^2 (w_i - a_t) (w_i - a_t)^{\top} \bSigma_i^2 }{ \|\bSigma_i (w_i - a_t) \|_2^2 }+ \bSigma_i^2  \right) \right).
\end{align}
Let us define the signal matrix $\bX_i$ as 
\begin{align}
\bX_i &=   \frac{ \bSigma_i^2 (w_i - a_t) (w_i - a_t)^{\top} \bSigma_i^2  }{ \|\bSigma_i (w_i - a_t) \|_2^2 }  \left( \sum_{p = 0}^k  2p c_p \| \bSigma_i (w_i - a_t)\|_2^{2p} \right) 
\\
&= \frac{ \bSigma_i^2 (w_i - a_t) (w_i - a_t)^{\top} \bSigma_i^2  }{ \|\bSigma_i (w_i - a_t) \|_2^2 }  \left( f'( \| \bSigma_i (w_i - a_t) \|_2) \| \bSigma_i (w_i - a_t) \|_2 \right)
\end{align}
and the noise matrix $\bY_i$ as
\begin{align}
\bY_i &=\bSigma_i^2 \left( \sum_{p = 0}^k  c_p \| \bSigma_i (w_i - a_t)\|_2^{2p} \right) 
\\
&= \bSigma_i^2 f(\| \bSigma_i (w_i - a_t)\|_2)
\end{align}
such that 
\begin{align}
\E[\bM] = \sum_{i=1}^k p_i (\bX_i + \bY_i).
\end{align}

For $j = \argmin\{\| \bSigma_i (w_i - a_t) \|_2) \}_{i = 1}^k$, let us denote 
$$
\beta := f'( \| \bSigma_j (w_j - a_t) \|_2) \| \bSigma_j (w_j - a_t) \|_2.
$$

Let us recall that $\veps^{(g)}$ is the error incurred when estimating $\{\| \bSigma_i (w_i - a_t) \|_2\}_{i = 1}^k$. $\veps^{(p)}$ is the error when constructing the coefficients of the polynomial (for sufficiently large  $\rho$ such that $\rho \geq  \max\{\| \bSigma_i (w_i - a_t) \|_2^2) \}_{i = 1}^k$ as we will show later in this proof). Thus, by Lemma~\ref{lem:cluster}, we know that 
\begin{align}
\| \bY_i \|_2& \leq \| \bSigma_i^2 \|_2 | f(\| \bSigma_i (w_i - a_t)\|_2)|
\\
& \leq \| \bSigma_i^2 \|_2(| f(\sigma_i)| + 2k \rho^k \left|\sigma_i - \| \bSigma_i (w_i - a_t)\|_2 \right| )
\\
&\leq  \| \bSigma_i^2 \|_2 ( \veps^{(p)} \beta + 4k \rho^k \veps^{(g)}).
\end{align} 
Similarly we have
\begin{align}
\|\bX_j\|_2 \geq \sigma_{\min}(\bSigma_j^2) \beta.
\end{align}
And we have $\beta \geq \left(\frac{\veps^{(p)}}{\rho} \right)^k - 8 k^2 \rho^k \veps^{(g)} \sigma^2$.

Notice that $ \min\{\| \bSigma_i (w_i - a_t) \|_2) \}_{i = 1}^k \leq  \min\{\| \bSigma_i (w_i ) \|_2) \}_{i = 1}^k$, which implies that $\| a_1 \|_2 \leq \sigma^4$. Therefore, we can take $\rho =O\left( \max\left\{2\sigma^{10}, \frac{1}{\veps} \right\}\right)$. Thus, by our choice of parameter, we know that for $\veps^{(e)} \leq \frac{1}{100 k} $, 
\begin{align}
\left\|\E[\bM] - \sum_{i= 1}^k p_i \bX_i \right\|_2 \leq \veps^{(e)}\| \bX_j \|_2/2.
\end{align}
Using the sample complexity bound Lemma~\ref{lem:gsb}, by our choice of $m$ we know that 
\begin{align}
\left\|\bM - \E[\bM] \right\|_2 \leq \veps^{(e)}\| \bX_j \|_2/2.
\end{align}
Thus, apply Lemma \ref{lem:k_SVD} on $\bM$ we know that 
\begin{align}
\frac{\|\bU_t^{\top} \bX_j \bU_t \|_2}{\|\bX_j\|_2} \geq 1 - \left(\veps^{(e)}k\right)^{1/3} \geq \frac{3}{4}.
\end{align}
Indeed, this also implies that 
\begin{align}
\frac{\|\bU_t^{\top} \bSigma_j^2 (w_j - a_t) \|_2}{\| \bSigma_j^2 (w_j - a_t)  \|_2 } \geq \frac{1}{2}
\end{align}
completing the proof.
\end{proof}

Now we can prove the main lemma regarding the per-iteration improvement of Algorithm~\ref{alg:one}.


\medskip
\noindent
\textbf{Lemma~\ref{lem:one_iter} (Coefficients)}
{\it
For every $t \in \{0, 1, \cdots, T - 1\}$ and $\delta > 0$, as long as $\sigma_t = \Omega ( \sigma \veps)$, then with probability at least $1 - \delta$,
$$
\sigma_{t + 1}^2 \leq \left(1 - \frac{1}{200 k \sigma} \right) \sigma_t^2.
$$
}
\medskip

\begin{proof}[Proof of Lemma~\ref{lem:one_iter}]
At $t$-th iteration let $j = \argmin \{ \|\bSigma_i (w_i - a_t) \|_2\}_{i = 1}^k$, we know that 
\begin{align}
\frac{\|\bU_t^{\top} \bSigma_j^2 (w_j - a_t) \|_2}{\| \bSigma_j^2 (w_j - a_t)  \|_2 } \geq \frac{1}{2}.
\end{align}

By definition, $v = \frac{\bU_t \gamma}{\| \bU_t \gamma \|_2}$ for $\gamma \in \mathcal{N}(0, \bI)$. Thus, using elementary calculation of Gaussian random variables, we have: with probability at least $1/4$, 
\begin{align}
\frac{v^{\top} \bSigma_j^2 (w_j - a_t) }{ \| \bSigma_j^2 (w_j- a_t)  \|_2 } \geq \frac{1}{10 \sqrt{k}}
\end{align}
which implies that
\begin{align}
\left\|\bSigma_j (w_j - a_t  - \eta  v) \right\|_2^2 &= \left\|\bSigma_j (w_j - a_t ) \right\|_2^2 - 2  \eta \langle \bSigma_j (w_j - a_t), \bSigma_j  v\rangle + \eta^2 \| \bSigma_j  v \|_2^2
\\
&=  \left\|\bSigma_j (w_j - a_t ) \right\|_2^2  - 2 \eta \langle \bSigma_j^2 (w_j - a_t),  v\rangle + \eta^2 \| \bSigma_j  v \|_2^2
\\
& \leq  \left\|\bSigma_j (w_j - a_t ) \right\|_2^2  - \frac{\eta}{5 \sqrt{k}} \| \bSigma_j^2 (w_j - a_t) \|_2 + \eta^2  \sigma.
\end{align}

Let $\eta = \frac{\| \bSigma_j^2 (w_j - a_t) \|_2 }{10 \sigma \sqrt{k}}$. Then we know that 
$$
\left\|\bSigma_j (w_j - a_t  - \eta v) \right\|_2^2 \leq \left(1 - \frac{1}{100 k \sigma} \right)  \left\|\bSigma_j (w_j - a_t ) \right\|_2^2.
$$

Thus, since we can estimate $\left\|\bSigma_j (w_j - a_t  - \eta  v) \right\|_2$ up to accuracy $\veps/(k\sigma)$ using the algorithm proposed in~\citep{moitra2010settling}, as long as $\sigma_{t} = \Omega\left( \sigma\veps \right)$, we will have that $\sigma_{t + 1}^2 \leq  \left(1 - \frac{1}{200 k \sigma} \right)\sigma_t^2$. 
\end{proof}

This immediately leads to the main lemma regarding the output of Algorithm~\ref{alg:one}.


\medskip
\noindent
\textbf{Lemma~\ref{lem:warmstart}}
{\it
With probability at least $1-\delta$, $\min_i \|w_i - a_T\|_2 \le O(\sigma^2\veps)$.
}
\medskip

\begin{proof}[Proof of Lemma~\ref{lem:warmstart}]
By Lemma~\ref{lem:one_iter}, and by the choice of the parameters in the algorithm,
$
  \sigma_T \le O(\sigma \veps).
$
Then for $j = \min_i \{\|\Sigma_i (w_i - a_T)\|_2\}$ we have 
$
  \|\Sigma_j (w_j - a_T)\|_2 \le O(\sigma \veps)
$
and thus
$ 
  \|w_j - a_T\|_2 \le O(\sigma^2 \veps).
$
\end{proof}

\section{Proof for Learning One of the Weights from Warm Start} \label{sec:proof_gd}

Without loss of generality, let us assume that we have an $v$ such that $\| v - w_1 \|_2$ is reasonably small. We will show that the update rule used in the algorithm can recover $w_1$ up to error $\veps$ with this $v$. 
%
It is equivalent to (the empirical version of) the gradient descent update to minimize the following concave objective function:
$$
  g(v) = \E \left [\log (|\alpha - \langle v, x \rangle| + \zeta)  \right].
$$



\medskip
\noindent
\textbf{Lemma~\ref{lem:gradient} (Gradient descent)}
{\it
Suppose there exists $i \in [k]$ such that $\| w_i - v \|_2 \le \zeta / \sigma$. Then with high probability, Algorithm~\ref{alg:gd} outputs a vector $v^{(T)}$ such that
$
  \| w_{i} - v^{(T)} \| \le \veps.
$
}
\medskip

\begin{proof}[Proof of Lemma~\ref{lem:gradient}]
First, suppose we have the gradient on the expectation, i.e., we have $\nabla g(v^{(t)})$. For this gradient descent update rule,
by Lemma~\ref{lem:inverse_gaussian}, we know that 
\begin{align*}
 \left\langle - \nabla g(v^{(t)}), w_1 - v^{(t)} \right\rangle 
  & = \E \left[\frac{\sign(\alpha - \langle v^{(t)}, x \rangle) \langle w_1 - v^{(t)}, x \rangle}{|\alpha - \langle v^{(t)}, x \rangle| + \zeta}  \right] 
	\\
	& = p_1 \E_{y \sim \mathcal{N}(0, 1)} \E \left[\frac{\sign(\langle \bSigma_1 (w_1 - v^{(t)}), y \rangle) \langle \bSigma_1 (w_1 - v^{(t)}), y \rangle}{|\langle \bSigma_1(w_1 - v^{(t)}), y \rangle| + \zeta}  \right]
  \\
  & ~~ + \sum_{j = 2}^k p_j \E_{y \sim \mathcal{N}(0, 1)} \E \left[\frac{\sign(\langle \bSigma_j (w_j - v^{(t)}), y \rangle) \langle \bSigma_j (w_1 - v^{(t)}), y \rangle}{|\langle \bSigma_j(w_j - v^{(t)}), y \rangle| + \zeta}  \right]
  \\
  & \geq \frac{1}{4} p_1 \frac{\| \bSigma_1 (w_1 - v^{(t)})\|_2 }{\|\bSigma_1 (w_1 - v^{(t)}) \|_2+ \zeta} - \sum_{j = 2}^k p_j \frac{\|\bSigma_1 (w_1 - v^{(t)}) \|_2} {\| \bSigma_j (w_j - v^{(t)})\|_2 }.
\end{align*}

Note that our assumption on $\zeta$ satisfies that
\begin{align} \label{eqn:grad_condition}
  \|\bSigma_1 (w_1 - v^{(t)}) \|_2 \leq \zeta, \quad \|\bSigma_j (w_j - v^{(t)}) \|_2 \geq 32 \zeta / p_{\min}, j\neq 1,
\end{align}

Therefore, a direct calculation shows that 
$$
  \left\langle -\nabla g(v^{(t)}), w_1 - v^{(t)} \right\rangle 
	\geq \frac{p_{\min}}{32 } \frac{\|\bSigma_1 (w_1 - v^{(t)}) \|_2}{\zeta} 
	\geq   \frac{p_{\min} \| w_1 - v^{(t)} \|_2}{32 \zeta}.
$$

However, we only have the empirical version of the gradient given as
$$
 -\tilde\nabla g(v^{(t)}) = \E_{(x_\ell,\alpha_\ell)} \nabla g_\ell(v), \mbox{~where~} -\nabla g_\ell(v^{(t)}) = \frac{\sign(\alpha_\ell - \langle v^{(t)}, x_\ell \rangle)}{|\alpha_\ell - \langle v^{(t)}, x_\ell \rangle| + \zeta} x_\ell.
$$

To apply concentration bound on the empirical version, we know that for for every example $(x,\alpha)$, 
$$
  \left\|\frac{\sign(\alpha - \langle  v^{(t)}, x \rangle)}{|\alpha - \langle v^{(t)}, x \rangle| + \zeta} x  \right\|_2 \leq \frac{\|x\|_2}{\zeta}. 
$$

Moreover, we know that the true gradient satisfies
$$
  \left\langle - \nabla g(v^{(t)}) , \frac{ w_1 - v^{(t)}}{\| w_1 - v^{(t)}\|_2}  \right\rangle \geq \frac{p_{\min}}{32 \zeta} 
$$

For every example $(x, \alpha)$, we have
$$
  \left|\left\langle  \frac{\sign(\alpha - \langle v^{(t)}, x \rangle) x }{|\alpha -  \langle v^{(t)}, x \rangle| + \zeta}   , \frac{ w_1 - v^{(t)}}{\| w_1 - v^{(t)}\|_2} \right\rangle \right|
	\leq 
	\frac{\left|\left\langle\frac{ w_1 - v^{(t)}}{\| w_1 - v^{(t)}\|_2}, x  \right\rangle\right|}{\zeta}.
$$

Using an elementary concentration bound of Gaussian random variables, we know that with $\text{poly}\left(\frac{1}{\zeta}, \frac{1}{p_{\min}}, \sigma\right)$ examples, the estimated gradient $\tilde{\nabla} g(v^{(t)})$ satisfies with high probability that
$$
  \| \tilde{\nabla } g(v^{(t)}) \|_2 \leq \frac{4 \sqrt{d}}{\zeta}, 
	\quad 
	\left\langle -\tilde{ \nabla} g(v^{(t)}) , \frac{ w_1 - v^{(t)}}{\| w_1 - v^{(t)}\|_2}  \right\rangle \geq \frac{p_{\min}}{64 \zeta}.
$$

Then when $\eta_t = c\frac{\zeta p_{\min} \| w_1 - v^{(t)}\|_2}{ d}$ for a sufficiently small constant $c>0$, and using the assumptions on $v^{(0)}$ and $\Delta$ to satisfy the condition (\ref{eqn:grad_condition}), by induction, we have
$$
 \| w_1 - v^{(t + 1)}\|_2^2 \leq \left( 1 - \Omega\left( \frac{p_{\min}^2}{ d} \right)\right) \| w_1 - v^{(t)} \|_2^2
$$
completing the proof.
\end{proof}

\section{Proof for Learning All the weights} \label{sec:proof_all}

\noindent
\textbf{Theorem~\ref{thm:main} (Main)} 
{\it
Assume the model~(\ref{def:mlr}) and assumptions \textbf{(A1)}-\textbf{(A3)}. Then Algorithm~\ref{alg:mlr} takes $N=d \log\left(\frac{d}{\veps}\right)\cdot \left(\frac{\sigma}{\Delta p_{\min}} \right)^{O(k)} +  \left( \frac{\sigma }{\Delta p_{\min} \veps} \right)^{O(k^2)}$ data points and in time $Nd \cdot \textrm{polylog}(k, d, \sigma, \frac{1}{\Delta}, \frac{1}{p_{\min}}, \frac{1}{\veps}) $  outputs a set of vectors $\{v_i\}_{i=1}^k$ that with high probability satisfy
$$
  \|v_i  - w_{\pi(i)} \|_2 \le \veps, \forall i \in [k], ~\mbox{for some permutation $\pi$}.
$$  
}

\begin{proof}[Proof of Theorem~\ref{thm:main}]
The theorem follows from Lemma~\ref{lem:gradient} and Lemma~\ref{lem:one_iter}, the guarantees for the two subroutines used. Note that we recovers each weight up to $\veps_g \le \left(\frac{p_{\min} \Delta}{\sigma d}\right)^{\Omega(k^2)}$. 
Therefore, only a $\left(\frac{p_{\min} \Delta}{\sigma d}\right)^{\Omega(k^2)}$ fraction of data points from this component are not removed, and only a $\left(\frac{p_{\min} \Delta}{\sigma d}\right)^{\Omega(k^2)}$ fraction of data points from other components get removed. These only causes polynomially small errors to the quantities computed in later steps and can be tolerated by our analysis. 
\end{proof}

\section{Tools}

We shall use the following bounds on the Gaussian moments and it's concentration.

\begin{lemma}
Let $g \sim \mathcal{N}(0, \bI)$, then for every unit vector $w$, we have that for every non-negative integer $p$, 
$$
 \E\left[\langle w, g \rangle^{2p} g g^{\top}\right] =  (2p + 1)!! w w^{\top} + (2p - 1)!! (\bI - w w^{\top} ).
$$
\end{lemma}

Using a standard Matrix Bernstein bound, we can get:
\begin{lemma}[Gaussian sample bound]\label{lem:gsb}
Let $g \sim \mathcal{N}(0, \bSigma^2)$, let $g_1, \cdots, g_m$ be $m$ independent samples of $g$.  Then for every vector $w$ and every non-negative integer $p$ and every $\delta > 0$, we have that 
\begin{align}
\Pr\left[\left\|\frac{1}{m}\sum_{i = 1}^m \langle w, g_i \rangle^{2p} g_i g_i^{\top} - \E\left[\langle w, g \rangle^{2p} g g^{\top}\right] \right\|_2 = \Omega\left( \sqrt{  \frac{\| \bSigma w \|_2^{4p} \left\| \bSigma \right\|^4_2 d \log \frac{1}{\delta}}{ m} } \right)\right] \leq \delta
\end{align}
\end{lemma}

The following lemma gives an estimation of a (modified) inverse Gaussian, which is used for analyzing the gradient descent step of our algorithm.

\begin{lemma}\label{lem:inverse_gaussian}
Suppose $y \sim \mathcal{N}(0, \bI)$. For every $\zeta > 0$, for every vectors $a, b \in \mathbb{R}^d$, with $\rho = \frac{\langle a, b \rangle}{\|a\|_2 \|b\|_2}$,
$$ 
  \frac{1}{4}\frac{\rho \|a\|_2}{\zeta + \|b\|_2}  \leq \E\left[\frac{\sign(\langle b, y \rangle)\langle a , y \rangle}{| \langle b, y \rangle |+ \zeta} \right]  \leq \frac{\rho \|a\|_2}{\|b\|_2} \leq \frac{ \|a\|_2}{\|b\|_2}.
$$
\end{lemma}

\begin{proof}[Proof of Lemma \ref{lem:inverse_gaussian}]
Without loss of generality assume $b = \|b\|_2 e_1$ and $a =  \|a\|_2 (\rho e_1 + \sqrt{1 - \rho^2} e_2 )$. Then 
\begin{align*}
\E\left[\frac{\sign(\langle b, y \rangle)\langle a , y \rangle}{| \langle b, y \rangle |+ \zeta} \right]&= \E\left[\frac{ \|a\|_2 (\rho y_1+ \sqrt{1 - \rho^2} y_2 ) \sign(y_1)}{ \| b\|_2 |y_1| + \zeta} \right] 
\\
& = \rho \|a\|_2 \E\left[ \frac{|y _1|}{\|b\|_2 |y_1 |+ \zeta} \right]
\end{align*}
We know that 
$$  
  \frac{|y _1|}{\|b\|_2 |y_1 |+ \zeta} \leq \frac{1}{\|b\|_2},
$$
and when $|y_1| \ge 1$
$$
  \frac{|y _1|}{\|b\|_2 |y_1 |+ \zeta} \ge \frac{1}{\zeta + \|b\|_2}.
$$
Therefore, we have
$$ \frac{1}{4}\frac{\rho \|a\|_2}{\zeta + \|b\|_2}  \leq \E\left[\frac{\sign(\langle b, y \rangle)\langle a , y \rangle}{| \langle b, y \rangle |+ \zeta} \right]  \leq \frac{\rho \|a\|_2}{\|b\|_2}.
$$
where the first inequality follows from $\E[1_{|y_1| \ge 1}] \ge 1/4$.
\end{proof}

We will also need the Gap-Free Wedin Theorem from~\citep{allen2016lazysvd}. 

\begin{lemma}[Gap-Free Wedin Theorem, Lemma B.3 in~\citep{allen2016lazysvd}] \label{lem:gapfree_wedin}
For $\veps \ge 0$, let $A, B$ be two PSD matrices such that $\|A - B\|_2 \le \veps$. For every $\mu \ge 0, \tau > 0$, let $U$ be the column orthonormal matrix consisting of eigenvectors of $A$ with eigenvalue $\le \mu$, let $V$ be column orthonormal matrix consisting of eigenvectors of $B$ with eigenvalue $\ge \mu + \tau$, then we have:
$$
  \| U\top V\| \le \frac{\epsilon}{\tau}.
$$
\end{lemma}

\end{document}